\documentclass[conference]{IEEEtran}
\ifCLASSINFOpdf
\else
\fi
\usepackage{color}
\usepackage{graphicx}
\usepackage{amssymb,amsfonts,amsmath}
\usepackage{amsthm}
\usepackage{float}
\usepackage{epsfig}
\usepackage{comment}
\usepackage{verbatim}
\usepackage{fullpage}
\usepackage[all,dvips,arc,curve,color,frame]{xy}
\usepackage{algorithm}
\usepackage{algorithmic}
\usepackage{multirow}
\usepackage{bbm}

\begin{document}
%
\title{Structural Similarity and Distance in Learning}

\author{\IEEEauthorblockN{Joseph Wang}
\IEEEauthorblockA{Dept. of Electrical and\\Computer Engineering\\
Boston University\\
Boston, MA 02215\\
Email: joewang@bu.edu}
\and
\IEEEauthorblockN{Venkatesh Saligrama}
\IEEEauthorblockA{Dept. of Electrical and\\Computer Engineering\\
Boston University\\
Boston, MA 02215\\
Email: srv@bu.edu}
\and
\IEEEauthorblockN{ David A. Casta\~{n}\'{o}n}
\IEEEauthorblockA{Dept. of Electrical and\\Computer Engineering\\
Boston University\\
Boston, MA 02215\\
Email: dac@bu.edu}
}


%


\newcommand{\fix}{\marginpar{FIX}}
\newcommand{\new}{\marginpar{NEW}}

\newcommand{\argmax}{\operatornamewithlimits{argmax}}
\newcommand{\argmin}{\operatornamewithlimits{argmin}}
\newcommand{\indicator}[1]{\mathbbm{1}_{\left[ {#1} \right] }}

\newcommand{\comments}[1]{}
\newtheorem{thm}{Theorem}
\newtheorem{prop}[thm]{Proposition}
\newtheorem{lem}[thm]{Lemma}
\newtheorem{cor}[thm]{Corollary}
\newtheorem{definition}[thm]{Definition}
\newtheorem{example}[thm]{Example}

\maketitle

\begin{abstract}
We propose a novel method of introducing structure into existing machine learning techniques by developing structure-based similarity and distance measures. To learn structural information, low-dimensional structure of the data is captured by solving a non-linear, low-rank representation problem. We show that this low-rank representation can be kernelized, has a closed-form solution, allows for separation of independent manifolds, and is robust to noise. From this representation, similarity between observations based on non-linear structure is computed and can be incorporated into existing feature transformations, dimensionality reduction techniques, and machine learning methods. Experimental results on both synthetic and real data sets show performance improvements for clustering, and anomaly detection through the use of structural similarity.
\end{abstract}


%
\IEEEpeerreviewmaketitle

\section{Introduction}




The notion of distance, or more generally, similarity between observations, is at the root of most learning algorithms such as manifold learning, unsupervised clustering, semi-supervised learning, and anomaly detection. Most methods, at a basic level, are based on some function of Euclidean distance, such as the radial basis functions prevalent in supervised classification. Graph-based learning methods employ Euclidean distances to describe local neighborhoods for observations. K-nearest neighbors and $\epsilon$-neighborhoods based on Euclidean distances are used in manifold learning (Isomap \cite{ISOMAP} and LLE \cite{local_linear_embedding}), spectral clustering \cite{Ng_spectral_clustering}, anomaly detection algorithms \cite{Manqi_AD,hero_AD} and in label propagation algorithms for semi-supervised learning \cite{Label_Propagation}.

\begin{figure}[ht]
\centering
\begin{center}
\centerline{\includegraphics[width=3in]{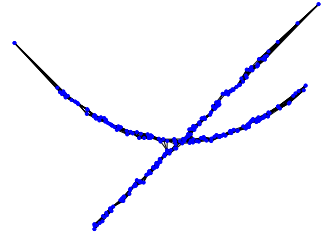}}
\vspace{-0.15in}
\caption{K-nearest neighbor graph constructed on densely-sampled 2-D simulated data set. With densely sampled data points, the graph is representative of the underlying structure.}
\label{intro_dense_fig}\end{center}
\end{figure}

In many cases, notably sparsely sampled sets of data, Euclidean neighborhoods are not sufficient to represent underlying structure. Consider data drawn from two independent structures. Ideally, a graph would capture the structure of the data with minimal connection between observations on separate structures. For densely sampled data on the manifolds, as shown in Fig. \ref{intro_dense_fig}, local Euclidean neighborhoods tend to lie on the same independent manifold, and therefore the K-nearest neighbor graph using Euclidean distance captures the structure of the data. However, when the data is sparsely sampled, neighboring points in the Euclidean sense fail to lie on the same structure, as shown in Fig. \ref{intro_sparse_fig}. We propose a new notion of similarity that accounts for global structure as well as local Euclidean neighborhoods. By using this new notion of similarity, we can define neighborhoods dependent on both Euclidean distance as well as structural similarity.

\begin{figure}[ht]
\centering
\begin{center}
\centerline{\includegraphics[width=3in]{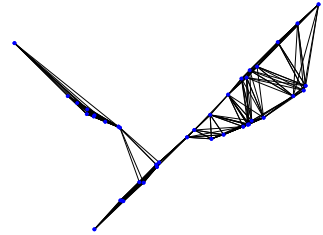}}
\vspace{-0.15in}
\caption{K-nearest neighbor graph constructed on sparsely-sampled 2-D simulated data set. The graph is not representative of the underlying structure of the data, as Euclidean neighborhoods include points lying on both independent structures.}
\label{intro_sparse_fig}\end{center}
\end{figure}

In many situations observations are well described by a single or a union of multiple low-dimensional manifolds. Most methods deal with this scenario by patching together local neighborhoods or local linear subspaces. These local neighborhoods are generally based on local Euclidean balls around each data sample. In cases when observations are sparsely sampled, are noisy, or lie on places where manifold has high curvature, these local approximations could be inaccurate and may not characterize the underlying manifold. Sparse sampling commonly arises in high-dimensions and measurement noise is common to many signal processing applications.

Unfortunately, Euclidean distance does not capture low-dimensional structure of observations unless the manifold is highly sampled. We propose a new approach to more accurately describe local neighborhoods by explicitly incorporating low-dimensional manifold structure of data. We present a complementary method of incorporating similarity into existing machine learning techniques. Our approach can be used in conjunction with dimensionality reduction, feature transformation, and kernel selection techniques to incorporate structure.

In order to capture the low-dimensional structure of the data, we first solve a regularized low-rank representation problem on the observed data. We derive a computationally-efficient closed-form solution solution which allows for handling large sets of observations. The resulting solution produces a low rank matrix, $Z$. Each column of the matrix, $Z$, is associated with a data sample and so the columns of the matrix represent a transformation of data into a new coordinate space. We show that these techniques can be extended to non-linear settings by demonstrating that the low-rank representation problem can be kernelized. We then present algorithms for estimating low rank representation for test samples that conform with the representation for training data.

The problem of determining an approximate representation for data using low-rank subspaces has recently drawn significant interest in the field of matrix completion problems \cite{candes_recht,matrix_completion_from_a_few_entries}. Methods for low-rank representations (LRR) of data drawn from multiple sources belonging to union of subspaces have also been developed \cite{LRR_rep,KLRR_older_tech_report,KLRR_tech_report}. Low-rank representations seek to segment data vectors such that each segmented collection belongs to a low-dimensional linear subspace. Low-rank representation of data is related to traditional simultaneous sparse representation techniques \cite{tropp_eurasip} with important differences. The objective in simultaneous sparsity is to decompose data vectors so that they have a common basis in a dictionary. In both the rank-minimization and simultaneous sparsity problems, the goal is representation of data subject to a structural constraint. In comparison , we are not interested in exact representation of observations, but instead in embedding points in a linear plane, with the notion of low-rank structures as a means of defining neighborhoods. In particular, our resulting minimization resembles the problem posed by Liu et al. \cite{LRR_rep}. However, the problem posed in this paper can be solved in a computationally efficient manner, extended to nonlinear manifolds, and analyzed in the presence of noise.

We theoretically show that the resulting low rank matrix $Z$ has a block diagonal structure, with each block having low rank. In the linear setting, this implies that the data space is automatically decomposed into multiple linear/affine patches. Decomposition into non-linear patches follows from kernelized extensions. This new representation is purely geometric, applies to single or multiple manifolds, and does not require specifying a metric on the manifold. It is adaptive in that it does not require pre-specification of the number of data points for each patch. It is global in that the matrix $Z$ is obtained by solving the low-rank representation on the entire data set. This block diagonal structure is essentially preserved in noisy situations or when the underlying manifold can only be approximated by linear or kernel representations.

This new representation $Z$ can be used in several ways. It can be used as samples from a new low-dimensional feature/observation space. Nearest neighbors for data samples or similarity between different data samples can be derived in this new representation. Alternatively, structural similarity can be computed between observations. Consequently, this new representation can be viewed as a pre-processing step not only for most machine learning algorithms, but also as a pre-processing step for other pre-processing steps, such as graph construction, that are commonly undertaken for machine learning.  We then present a number of simulations on a wide variety of data sets for a wide range of problems including clustering, semi-supervised learning and anomaly detection. The simulations show remarkable improvement in performance over conventional methods.


\section{Structured Similarity and Neighborhoods}
\label{s.klrr}
\subsection{Low-Rank Data Transformation}
\label{s.transform}
Consider a set of observations, $X = [x_1, \, x_2,\, \ldots x_n]$, where $x_i \in \mathbb{R}^{d \times 1}$, approximately embedded on multiple independent, low-dimensional manifolds. Our goal is to discover these manifolds using by using techniques to learn low-rank representations of the data. In the case where the observations are embedded on linear subspaces, the low-rank representation (LRR) problem can be formulated as:
\begin{align}
\label{noise_free_rank_constraint}
Z=\min_Z \|X-XZ\|_{F}^{2} \\
\text{s.t.  } \text{Rank}(Z)=R  \notag
\end{align}
where $\|\cdot\|_{F}$ is the Frobenius norm, and the solution, $Z$, is the minimum squared-error linear embedding on a $R$-dimensional subspace. Relaxing the constraint in \eqref{noise_free_rank_constraint}, the minimization can be equivalently written:
\begin{align}
\label{noise_free_rank_min}
\min_Z \|X-XZ\|_{F}^{2} +\lambda\cdot\text{Rank}(Z)
\end{align}
Optimizing the rank of a matrix is a non-convex, combinatorial optimization. The convex relaxation of rank, the nuclear norm, is substituted, resulting in the convex optimization:
\begin{align}
\label{LRR}
\min_Z \|X-XZ\|_{F}^{2} +\lambda\|Z\|_{*}
\end{align}
This related problem was originally posed as a subspace segmentation method by Liu et al. \cite{LRR_rep}, who minimize the $\ell_2/\ell_1$ embedding error. A Kernelized Low-Rank Representation (KLRR) formulation of the problem naturally follows for the case where data is embedded on nonlinear subspaces:
\begin{align}
\label{KLRR}
\min_Z \frac{1}{2}\|\phi(X)-\phi(X)Z\|_{F}^{2} +\lambda\|Z\|_{*}
\end{align}
where $\phi(\cdot)$ is an expanded basis function with an associated kernel function, $K(i,j)=\phi(i)^{T}\phi(j)$. The form and parameters of the function $\phi(\cdot)$ are an assumption on the structure of the observations. Ideally, $\phi(\cdot)$ is chosen such that all observations are well approximated in the expanded basis space with a linear low-dimensional approximation while still maintaining the relationship between observations. As in all kernel methods, the accuracy of the approximation of the manifold is dependent on the ability of the kernel to fit the data. We refer only to the kernelized problem, as the linear problem is a specific case, where $\phi(X)=X$ and $K(X,X)=X^{T}X$.

\begin{thm}
\label{klrr_cf_solution_thm}
For $\lambda\geq 0$, the KLRR problem \eqref{KLRR} is minimized by the representation:
\begin{align}
\label{lrr_cf_sol_eqn}
Z^{*}=UD_{\lambda}U^{T}
\end{align}
where the singular vectors, $U$, are found by the singular value decomposition of the kernel matrix $K(X,X)=\phi(X)^{T}\phi(X)=UDU^{T}$, and $D_{\lambda}$ is the diagonal matrix defined
\begin{align}
\label{soft_thresholding_operator}
D_{\lambda}(i,i)=
\begin{cases}
\mathbf{1-\frac{\lambda}{\sigma_{i}}} & \text{if $\sigma_{i}>\lambda$}
\\
\mathbf{0} & \text{otherwise}
\end{cases}
\end{align}
and $\sigma_{i}$ is the $i$th singular value of the kernel matrix.
\end{thm}
\begin{proof}
The nuclear norm is unitarily invariant, so substituting $\hat{Z}=U^{T}ZU$, where $U$ is the matrix of singular vectors of the kernel matrix produces an equivalent minimization
\begin{align}
U^{T}ZU=\argmin_{\hat{Z}} \frac{1}{2}\|VDU^{T}-VD\hat{Z}U^{T}\|_{F}^2 +\lambda\|\hat{Z}\|_{*}
\end{align}
where $VDU^{T}$ is the singular value decomposition of the matrix $\phi(X)$. The Frobenius norm is also unitarily invariant, and therefore pre-multiplying and post-multiplying the argument of the Frobenius norm by the matrices $V^{T}$ and $U$ respectively produces the following minimization:
\begin{align}
U^{T}ZU=\argmin_{\hat{Z}} \frac{1}{2}\|D-D\hat{Z}\|_{F}^2 +\lambda\|\hat{Z}\|_{*}
\end{align}
In the above minimization, $D$ is a diagonal matrix, resulting in $\hat{Z}$ also being a diagonal matrix, as any off diagonal elements will increase the value of the cost function in both the Frobenius norm and nuclear norm terms. With a diagonal structure, the Frobenius norm is equivalent to the $\ell_{2}$ norm of its weighted diagonal terms, and the nuclear norm of $\tilde{Z}$ is equivalent to the $\ell_{1}$ norm on its diagonal.
\begin{align}
U^{T}ZU=D_{\lambda}=\argmin_{\hat{Z}} \frac{1}{2}\|diag\left(D(I-\hat{Z})\right)\|_{2}^2 \notag\\
+\lambda\|diag(\hat{Z})\|_{1}
\end{align}
The solution to this minimization is given by the soft-thresholding operator, which gives the closed-form solution in the theorem statement.
\end{proof}

The closed-form solution to the KLRR problem has previously been shown in \cite{KLRR_tech_report}, with the closed=form solution to the linear low-rank representation problem shown in \cite{DBLP:conf/cvpr/FavaroVR11}.

From Theorem \ref{klrr_cf_solution_thm}, low-rank representations of high-dimensional expanded basis space observations can be computed efficiently using only kernel functions. Additionally, this solution has near block-diagonal structure for sets of observations existing on independent subspaces in the expanded basis space.

\begin{thm}
\label{block_diagonal_structure}
Given observations lying on independent subspaces in the expanded basis space, the low-rank representation is near block-diagonal, with elements off-block-diagonal bounded:
\begin{align}
z_{ij}\leq \lambda\sqrt{\sum_{i: \sigma_i>\lambda }\left(\frac{1}{\sigma_i}\right)^{2}}
\end{align}
where $x_i$ and $x_j$ lie on independent manifolds.
\end{thm}

\begin{proof}
Consider the case of $\phi(X)=\begin{bmatrix}\phi(X_{1}) & \phi(X_{2})\end{bmatrix}$, where $\phi(X_{1})\in\mathbb{R}^{D \times n_1}$ and $\phi(X_{2})\in\mathbb{R}^{D \times n_2}$ span independent subspaces and have rank $r_1$ and $r_2$, $r_1+r_2<D$ and $\text{Rank}(\phi(X))=r_1+r_2$.

$\phi(X)$ is composed of observations lying on independent subspaces and therefore any point in $\phi(X_1)$ or $\phi(X_2)$ can be expressed as a linear combination of other points in the sets $\phi(X_1)$ or $\phi(X_2)$, respectively. Given the compact SVD, $\phi(X)=U_X\Sigma_XV_X^{T}$, where only singular vectors associated with non-zero singular values are included in the basis $U_X$ and $V_X$. the linear transformation:
\begin{align}
W=\Sigma_X^{-1}U_X^{T}\phi(X)=V_X^{T}
\end{align}
preserves the dependence structure of $\phi(X)$, resulting in the decomposition:
\begin{align}
W=\begin{bmatrix}B_1 & B_2 \end{bmatrix}\begin{bmatrix}\alpha_1 & 0\\ 0 & \alpha_2 \end{bmatrix}
\end{align}
where $B_1 \in \mathbb{R}^{r_1+r_2 \times r_1}$ and $B_2 \in \mathbb{R}^{r_1+r_2 \times r_2}$ are basis matrices and $\alpha_1 \in \mathbb{R}^{r_1 \times n_1}$ and $\alpha_2 \in \mathbb{R}^{r_2 \times n_2}$ are the representation associated with each independent subspace. Substituting the singular value decompositions $\alpha_1=U_1\Sigma_1V_1^T$ and $\alpha_2=U_2\Sigma_2V_2^T$, $W$ can be expressed:
\begin{align}
\label{sv_bd_decomp}
W=B'\begin{bmatrix}V_1^T & 0\\ 0 & V_2^T\end{bmatrix}
\end{align}
where $B'=\begin{bmatrix}B_1U_1\Sigma_1 & B_2U_2\Sigma_2 \end{bmatrix}$. The matrix $W$ is a set of singular vectors, so the following property must hold:
\begin{align}
WW^{T}=B'\begin{bmatrix}V_1^T & 0\\ 0 & V_2^T\end{bmatrix}\begin{bmatrix}V_1 & 0\\ 0 & V_2\end{bmatrix}B'^{T}=B'B'^{T}=I
\end{align}
Therefore, $B'$ is an orthonormal matrix.

The solution to the KLRR problem can be expressed
\begin{align}
\label{noise_free_approx_sol}
Z^{*}=V_X(I-\lambda \Sigma_X^{-1})V_X^{T}=V_XV_X^{T}-V_X\lambda \Sigma_X^{-1}V_X^{T}\notag\\
=\begin{bmatrix}V_1 & 0\\ 0 & V_2\end{bmatrix}B'^{T}B'\begin{bmatrix}V_1^T & 0\\ 0 & V_2^T\end{bmatrix}-V_X\lambda \Sigma_X^{-1}V_X^{T}
\end{align}
The term $V_XV_X^T$ is block diagonal, so the inner product between representations lying on separate independent subspaces can be bounded by bounding the off-block-diagonal elements:
\begin{align}
\|V_X\lambda \Sigma_X^{-1}V_X^{T}\|_{\infty}\leq\|V_X\lambda \Sigma_X^{-1}V_X^{T}\|_{F}=\lambda \|\Sigma_X^{-1}\|_{F}
\end{align}
\end{proof}
From this structure, we construct measures of similarity that result in large distance between observations on separate manifolds independent of Euclidean distance, resulting in separation of independent manifolds.

For reliable performance, the near block-diagonal structure of the KLRR matrix should remain in the presence of small perturbations. To demonstrate the robustness to noise, we bound the Frobenius norm of the off-diagonal elements when the kernel matrix is perturbed, guaranteeing near block-diagonal structure in the presence of noise.

\begin{thm}
\label{perturbed_block_diagonal}
Consider a perturbed kernel matrix
\begin{align}
\widetilde{K}(X,X)=K(X,X)+E
\end{align}
where $K(X,X)$ has a rank $r$ composed of two independent subspaces. The perturbed KLRR, $\widetilde{Z}$, is found using the matrix $\widetilde{K}(X,X)$. Define the matrix $N$ to be the off diagonal blocks of the matrix $\widetilde{Z}$, such that for all $z_{ij} \in N$, the observations $x_i$ and $x_j$ lie on independent manifolds. Then the matrix $N$ is bounded:
\begin{align}
\|N\|_{F} \leq \frac{4\sqrt{2}\|E\|_{F}}{\sigma_r-\sigma_e}
\end{align}
where $\sigma_e$ is the largest singular value of the matrix $E$.

\end{thm}
The proof of this theorem is omitted due to length limitations. The bound is derived by bounding the canonical angle between perturbed eigenvectors, using Thm V.4.1 of Stewart and Sun \cite{pert_theory_stewart_sun}.

Theorem \ref{perturbed_block_diagonal} illustrates the robustness to noise of the low-rank representation. This is an adversarial bound, with the no restrictions placed on the structure of the perturbations. Given small perturbations such that $\sigma_e \in o(\sigma_r)$, the norm of the off diagonal elements can be bounded linearly with the norm of the perturbation. Therefore, small perturbations in the observations have small effects on the structure of the KLRR representation.

\subsection{Transformation for Test Observations}
\label{s.test}
Consider a new observation, $x_{test} \in \mathbb{R}^{d \times 1}$. In order to represent the new observation in the low-rank representation space, we extend the KLRR formulation from Equation~\eqref{KLRR}. We project the data onto the expanded basis span of the KLRR representation, $\phi(X)Z$, where $Z$ is the KLRR on the training set $X$. The minimum norm projection can be calculated as a function of kernel functions.
\begin{align}
\label{test_obs_rep}
z_{test}=Z\left(Z^{T}K(X,X)Z\right)^{-1}Z^{T}K(X,x_{test})
\end{align}
The representation, $z_{test}$, can be treated as a new sample from the low-rank feature space. This representation may be a poor representation of the original observation if it does not lie on the manifold. One measure of how well a new observation is characterized by a manifold is by projecting the new observation onto the low-dimensional manifold \eqref{test_obs_rep}, then measuring the residual energy of the observation in the expanded basis space.
\begin{align}
\label{manifold_residual}
r_{test}=\|\phi(x_{test})-\phi(X)z_{test}\|_{2}
\end{align}
This residual can also be calculated using only kernel functions and provides an evaluation of how well the low-rank representation fits a new observation.

%
%
%
%
%

\subsection{Structured Kernel Design for Supervised and Unsupervised Learning}
\label{s.kernel}
From the low-rank representation, we now present methods of constructing kernels. The low-rank transformation of the raw data offers possibilities for designing kernels that incorporate the underlying structure in the data.

In order to exploit this structure we consider some specific PSD kernels, which are basically the dot product, i.e.,
\begin{align}
\label{cos_sim}
w_{ij}=\frac{z_i^{T}z_j}{\|z_i\|\|z_j\|}
\end{align}
where $w_{ij}$ is the similarity between observations $i$ and $j$ and $z_i$ and $z_j$ are the $i$th and $j$th columns of $Z$, respectively. The value $w_{ij}$ is the magnitude of the cosine of the angle between the vectors $z_i$ and $z_j$. Given the near block-diagonal structure of the KLRR matrix, as shown in Theorem \ref{block_diagonal_structure}, observations lying on independent subspaces have a very small similarity.

One issue with this similarity function is that it is undefined if either $z_i$ or $z_j$ is identically zero. Consequently, we can define $w_{ij} = 0$ if either $z_i$ or $z_j$ is zero. With this convention it is possible to show that this similarity satisfies the properties of PSD, i.e.,
\begin{align*}
&\tilde K (z_i,z_j) = z_i^{T}z_j {1\over \|z_i\|}{1\over \|z_j\|}\\
&= \tilde K_1(z_i,z_j) g(z_i)g(z_j) = \tilde K_1(z_i,z_j) \tilde K_2(z_i,z_j)
\end{align*}
$\tilde K_2(z_i,z_j) = g(z_i)g(z_j)$ is a valid PSD kernel (since it represents a rank one structure). Now since both $\tilde K_1$ and $\tilde K_2$ are valid PSD kernels it follows that $\tilde K$ is a valid PSD kernel. The similarity proposed in \eqref{cos_sim} captures the structure of the observations, however, the scaling information is lost. In order to incorporate structural information while preserving spatial relationships in the observation space, we propose the PSD kernel:
\begin{align}
\label{scaled_cos_sim}
s_{ij}=K(x_i,x_j)=\frac{\langle z_i,z_j\rangle}{\|z_i\|\cdot\|z_j\|}e^{\frac{-\|x_i-x_j\|^2}{2\sigma^{2}}}
\end{align}
Two observations only have a large similarity if the observations lie on the same manifold and have a small geometric distance. If $x_i$ and $x_j$ lie on independent manifolds, from the structure of the KLRR matrix, the angle between the observations is small, and therefore the similarity, $s_{ij}$, is also small. Alternatively, if the observations lie on the same low-dimensional manifold, but have a large geometric distance, the exponential term drives the similarity to a small value.

\begin{figure}[ht]
\centering
\begin{center}
\centerline{\includegraphics[width=3in]{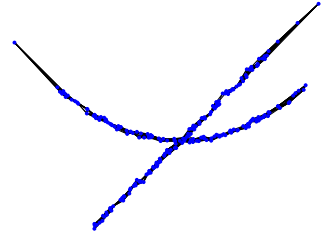}}
\vspace{-0.15in}
\caption{Graph for densely sampled 2-D simulated data set constructed by connecting the K-nearest structurally similar neighbors as defined by \eqref{scaled_cos_sim}.}
\label{dense_graph_example}\end{center}
\end{figure}

\begin{figure}[ht]
\centering
\begin{center}
\centerline{\includegraphics[width=3in]{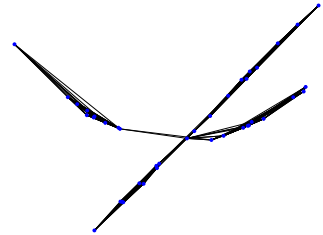}}
\vspace{-0.15in}
\caption{Graph constructed by connecting the K-nearest structurally similar neighbors as defined by \eqref{scaled_cos_sim}, which captures the structure of the data. The Euclidean K-nearest neighbor graph fails to capture the structure of the data, as shown in Fig. \ref{intro_sparse_fig}.}
\label{sparse_graph_example}\end{center}
\end{figure}

Figures \ref{dense_graph_example} and \ref{sparse_graph_example} demonstrate the effect of using structural similarity as opposed to Euclidean distance on the same sets of data as presented in Figs. \ref{intro_dense_fig} and \ref{intro_sparse_fig}. In the case of densely sampled manifolds, using both notions of similarity construct graphs that capture the underlying structure of the data, as shown in Fig. \ref{dense_graph_example}. However, when the data is sparsely sampled, the use of structural similarity allows the graph to accurately characterize the underlying structure of the data, as shown in Fig. \ref{sparse_graph_example}.

From the definition of similarity posed in \eqref{scaled_cos_sim}, a means of defining distance between observations follows:
\begin{align}
\label{vector_metric}
d(x_i,x_j)=\sqrt{s_{ii} + s_{jj} - 2 s_{ij}}
\end{align}
The metric \eqref{vector_metric} defines a new set of distances between observations combining both the structural similarity of the data as well as the Euclidean distance of the data. Two observations have a small distance if and only if the observations lie on the same low-dimensional manifold and have a small distance in the observations space.

%
%

\section{Manifold Anomaly Detection}
\label{s.AD}
The goal of our anomaly detection scheme is to define points not based on distance to nominal points, but instead based on distance to a low-dimensional manifold on which nominal points are embedded. In K-NNG and $\epsilon$-NNG approaches, the underlying manifold is modeled by dense sampling of data points, whereas our approach no longer requires dense sampling of data points, but instead structural assumptions on the data. For a set of nominal observations embedded on a manifold, we propose a method of anomaly detection based on p-value estimation \cite{Manqi_AD}.

Given a set of nominal training observations, $X$, a kernelized low-rank representation, $Z$, is found as described in Sectoin \ref{s.transform}. For a new test observation, $x_t$, a corresponding low-rank representation, $z_t$, is found through the update method described in Section \ref{s.test}. From these low-rank representations, the residual of the test observation is compared to the residuals of the labeled observations:
\begin{align}
\label{residual_p_value}
p_t=\frac{1}{n}\sum_{i=1}^{n} \mathbf{1}_{\bar{w}_t e^{-r_t} > \bar{w}_i e^{-r_i}}
\end{align}
where $r_i$ is the residual of the $i$th labeled observation, calculated as shown in \eqref{manifold_residual}, and $\bar{w}_t$ and $\bar{w}_i$ are the average angles cosine similarities of the representations as defined in \eqref{cos_sim}. The test observation is declared anomalous if $p_t > \alpha$. The proposed anomaly detection characterizes the nominal set by a nonlinear low-dimensional manifold and uses a measure of similarity to the manifold to determine if test observations are anomalous.

We modify our algorithm to simplify analysis. Assuming that $n$ is even, we divide the training set into two sets $S_1$ and $S_2$. We compute the KLRR as defined by \eqref{KLRR} for the set $S_1$ and compute representations for the set $S_2$ and the training sample, $x_t$, as defined by \eqref{test_obs_rep}. The p-value of the new observation is then estimated as follows:
\begin{align}
\label{analysis_p_value}
\hat{p}_t=\frac{1}{|S_2|}\sum_{i \in S_2} \mathbf{1}_{\bar{w}_t e^{-r_t} > \bar{w}_i e^{-r_i}}
\end{align}

The distribution of $\hat{p}_t$ approaches a uniform distribution over the range $[0,1]$ given that $x_t$ is nominal and drawn from the same distribution as the nominal observations, $X$. This follows from the lemma given by Zhao et al. \cite{localAD}:
\begin{lem}
\label{nested_p_value}
Given a function $G(x)$ has the nestedness property, that is, for any $t_1 > t_2$ we have ${x : G(x) > t1} \subset {x : G(x) > t2}$. Then $P_{x\sim f_n}\left(G(x) \geq G(\eta)\right)$ is uniformly distributed in [0, 1] if $\eta \sim f_n$.
\end{lem}
From this lemma, we can directly show that the distribution of $\hat{p}_t$ converges to a uniform distribution.

\begin{thm}
For a nominal test point, $x_t$, drawn from the same distribution as the labeled observations, $X$. $\hat{p}_t$ converges to a uniformly distributed random variable in the range $[0,1]$.
\end{thm}
\begin{proof}
This follows directly from Lemma \ref{nested_p_value}, as the function $G(x_i)=\bar{w}_i e^{-r_i}$ has the nested property.
\end{proof}
Therefore, the distribution of $p_t$ converges to a uniform as $n \rightarrow \infty$, and therefore the probability of false alarm converges to $\alpha$.

\section{Experimental Results}
\label{s.exp}
\subsection{Clustering}

\begin{figure}[ht]
\centering
\begin{center}
\centerline{\includegraphics[width=2.5in]{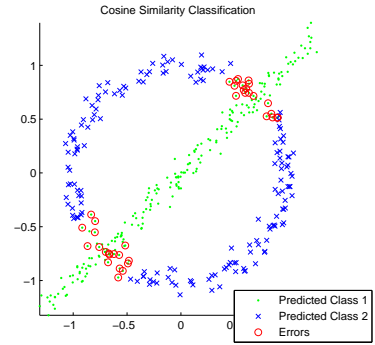}}
\vspace{-0.15in}
\caption{2-D simulated data set. Two classes are constructed, a line and circle, with 200 observations in each class and random gaussian noise added. Sample clustering results on simulated 2-D line-circle data set are shown using k-means clustering on the structural similarity defined in Equation \eqref{cos_sim}. }
\label{sim_clust_perf}\end{center}
\end{figure}

\begin{table}[h]
\begin{center}
\begin{tabular}{|c|c|c|c|c|c|}
\hline
\textbf{Data Set}& \textbf{Classes} & \textbf{Observation} & \textbf{Kernel} & \textbf{Similarity (W)}\\
\hline
Simulated & 2 & $49.4 \pm 0.1 \%$& $35.8 \pm 4.9 \%$ & $9.4 \pm 0.1 \%$ \\
\hline
Ionosphere & 2 & $28.8 \pm 0.1 \%$& $28.9 \pm 0.1 \%$ & $22.7 \pm 0 \%$ \\
\hline
Iris & 3 & $17.3 \pm 9.8 \%$ & $15.3 \pm 8.5 \%$ & $7.6 \pm 6.4 \%$\\
\hline
JAFFE & 10 & $29.5 \pm 4.4 \%$ & $25.5 \pm 5.0 \%$ & $ 12.7 \pm 5.3 \%$ \\
\hline
\end{tabular}
\label{Clustering_performance}
\caption{Average clustering error rates and standard deviation over 100 random initializations. Performance was compared for different representations: the original observations space (Observation), the expanded basis space (Kernel), and the cosine similarity space defined in \eqref{cos_sim} (Similarity (W)). For the JAFFE \cite{JAFFE_performance} and Iris databases \cite{iris_performance}, these performance rates are comparable to the best achieved results in literature and are achieved using an extremely simple algorithm (k-means clustering) on structured data.}
\end{center}
\vspace{-0.15in}
\end{table}


To evaluate performance, k-means clustering \cite{kmeans} was performed on representations of the data, with the results shown in Table \ref{Clustering_performance}. The k-means clustering algorithm was chosen as means to compare data representations due to its wide-spread use and lack of tuning parameters to be optimized. Initialization was performed by assigning observations to random clusters, with the error rates and standard deviations found for 100 random initializations. K-means clustering was tested on the data in the original feature space and in the expanded basis space, $\phi(X)$. For the simulated data, the expanded basis was generated from a 3rd order inhomogeneous polynomial kernel multiplied with a Gaussian RBF kernel. Note that this kernel does not perfectly transform the data to linear subspaces and therefore exact linear subspace recovery methods cannot be applied to this transform. For the Ionosphere, Iris, and JAFFE data sets, the expanded basis space was generated by Gaussian RBF kernels.

\begin{table}[h]
\begin{center}
\begin{tabular}{|c|c|c|c|c|c|}
\hline
\textbf{Data Set}& \textbf{Kernel} & \textbf{Similarity (W)} & \textbf{Struct. Kernel}\\
\hline
Simulated & $46.0 \pm 0 \%$ & $17.5 \pm 0 \%$ & $17.7 \pm 0\%$\\
\hline
Ionosphere &$35.9 \pm 0 \%$ & $22.5 \pm 0 \%$ & $22.8 \pm 0 \%$\\
\hline
Iris & $15.9 \pm 3.1 \%$ & $5.2 \pm 7.2 \%$ & $4.5 \pm 5.9 \%$\\
\hline
JAFFE &$17.14 \pm 5.2 \%$ & $ 13.5 \pm 5.8 \%$ & $13.7 \pm 5.3 \%$\\
\hline
\end{tabular}
\label{spectral_clustering_performance}
\caption{Average spectral clustering error rates and standard deviation over 100 random initializations. Performance was compared for different measures of similarity: the expanded basis space (Kernel), the cosine similarity space defined in \eqref{cos_sim} (Similarity (W)), and the structured kernel space defined in \eqref{scaled_cos_sim} (Structured Kernel). For the JAFFE \cite{JAFFE_performance} and Iris databases \cite{iris_performance}, these performance rates are comparable to the best achieved results in literature and are achieved using an extremely simple algorithm (k-means clustering) on structured data.}
\end{center}
\vspace{-0.15in}
\end{table}

Spectral clustering performance on an expanded basis space is compared to similarity measures incorporating structure in Table \ref{spectral_clustering_performance}. As with k-means clustering, inclusion of structure improved clustering performance in all example cases.

\subsection{Anomaly Detection}

We compare performance of the P-value estimation technique with the K-nearest neighbors graph (K-NN) method presented by Zhao et al. \cite{Manqi_AD} and a One-Class SVM \cite{oneclass_svm}. We evaluate performance on simulated data sets, the Ionosphere dataset \cite{UCI_database}, the USPS Digits data set \cite{elements_stat_learn}, and the JAFFE data set \cite{JAFFE_database}.

\begin{figure}[!t]
\begin{center}
\centerline{\includegraphics[width=2.5in]{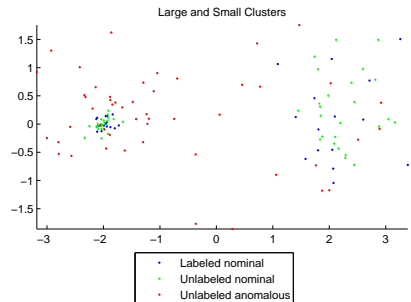}}
\vspace{-0.2in}
\caption{Example of simulated clusters data. Training was performed on 20 labeled nominal points (blue circles), and testing was performed on 50 unlabeled nominal points (green dots) and 50 unlabeled anomalous points (red crosses).}
\label{LS_data}\end{center}
\end{figure}

\begin{figure}[!t]
\begin{center}
\centerline{\includegraphics[width=2.5in]{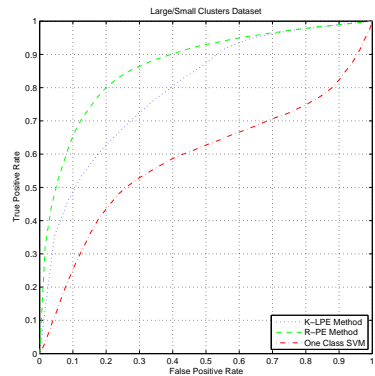}}
\vspace{-0.2in}
\caption{ROC curves averaged over 100 randomly generated data sets. Performance of p-value estimation using the KLRR residual is compared to p-value estimation using a Euclidean neighborhood (2nd nearest neighbor) and One-Class SVM with $\mu=0.5$.}
\label{LS_ROC}\end{center}
\end{figure}

The simulated clusters data set \ref{LS_data} consists of nominal data composed of two Gaussian distributions with different variances, and anomalous data drawn from a uniform distribution. 20 random nominal points were used to train the classifier, and performance was measured on a test set composed of 50 unobserved nominal points and 50 anomalous points, as shown in Fig. \ref{LS_data}. A Gaussian radial basis function kernel was used to approximate the manifold, and performance was averaged over 100 randomly generated data sets, with average performance shown in Fig. \ref{LS_ROC}.

\begin{figure}[!t]
\begin{center}
\centerline{\includegraphics[width=2.5in]{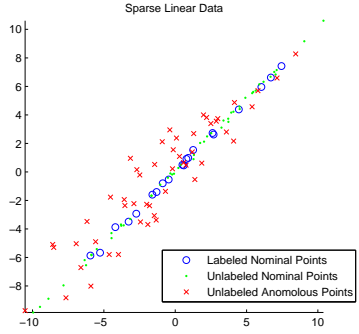}}
\vspace{-0.2in}
\caption{Example of simulated linear data. Training on 20 labeled nominal points (blue circles), testing on 50 unlabeled nominal points (green dots) and 50 unlabeled anomalous points (red crosses).}
\label{linear_data}\end{center}
\end{figure}

\begin{figure}[!t]
\begin{center}
\centerline{\includegraphics[width=2.5in]{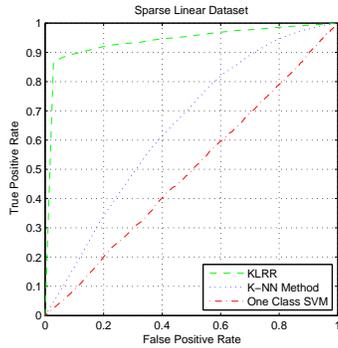}}
\vspace{-0.2in}
\caption{ROC curves averaged over 100 randomly generated data sets. Performance of p-value estimation using the KLRR residual is compared to p-value estimation using a Euclidean neighborhood (2nd nearest neighbor) and One-Class SVM with $\mu=0.5$.}
\label{linear_ROC}\end{center}
\end{figure}

The simulated linear data set was constructed of points generated form a linear subspace, with nominal points having small random perturbations and anomalous points having large perturbations. 20 random nominal points were used to train the classifier, and performance was measured on a test set composed of 50 unobserved nominal points and 50 anomalous points, as shown in Fig. \ref{linear_data}. 100 random data sets were generated, with an average performance shown in Fig. \ref{linear_ROC}. A linear low-rank representation of the labeled points was used to approximate the manifold.

\begin{figure}[!t]
\begin{center}
\centerline{\includegraphics[width=2.5in]{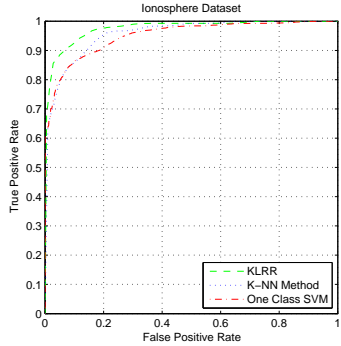}}
\vspace{-0.1in}
\caption{ROC curve using the proposed algorithm on the Ionosphere dataset. The ROC curve was generated by averaging results over 100 random sets. Performance using the KLRR residual and Euclidean distance (3rd nearest neighbor) for p-value estimation are shown, as well as the performance of a One-Class SVM with $\mu=0.5$.}
\label{ionosphere_ROC}\end{center}
\end{figure}

For the Ionosphere data set \cite{UCI_database}, 175 observations were labeled as nominal observations (drawn from the set which show evidence of structure in the ionosphere) and 30 observations were unlabeled for use as test data (drawn from both the "good" and "bad" observations). A gaussian radial basis function kernel was used, and performance was compared to anomaly detection using a K-nearest neighbor graph and a One-Class SVM, as shown in Figure \ref{ionosphere_ROC}.

For the JAFFE data set, 50 labeled nominal images were chosen from 3 random individuals (defined as nominal individuals) to construct the classifier. The test set was composed of 15 unobserved images randomly drawn from the nominal individuals and 100 anomalous images drawn from the other individuals. The performance using the KLRR residual was compared to the use of a K-nearest neighbor graph for p-value estimation \cite{Manqi_AD} and a One-Class SVM \cite{oneclass_svm}.
\begin{figure}[!t]
\begin{center}
\centerline{\includegraphics[width=2.5in]{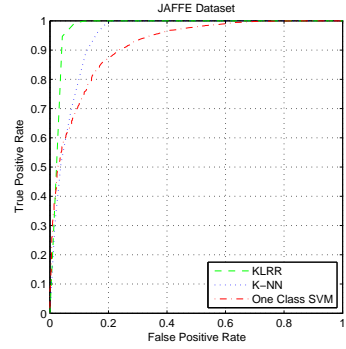}}
\vspace{-0.1in}
\caption{ROC curve using the proposed algorithm on the JAFFE dataset. The ROC curve was generated by averaging results over 100 random sets. Performance using the KLRR residual and Euclidean distance (3rd nearest neighbor) for p-value estimation are shown, as well as the performance of a One-Class SVM with $\mu=0.5$.}
\label{jaffe_ROC}\end{center}
\end{figure}
For the USPS Digits data set, 200 nominal images (the digit 8) were labeled, with 167 unlabeled images randomly drawn from the unobserved nominal images and 33 anomalous images drawn from the other digits. A Gaussian RBF was used to find the low-rank representation for both the USPS and JAFFE data sets, and the same kernel functions were used in the One-Class SVM.
\begin{figure}[!t]
\centering
\begin{center}
\centerline{\includegraphics[width=2.5in]{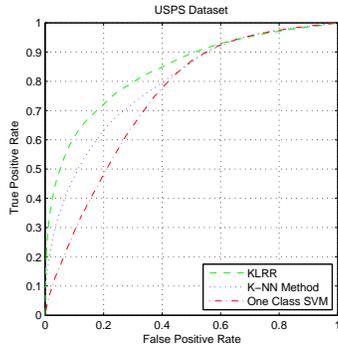}}
\vspace{-0.1in}
\caption{ROC curve on the USPS digits data set generated by averaging results over 100 random sets of labeled and unlabeled points. Performance using the KLRR residual and Euclidean distance (9th nearest neighbor) for p-value estimation are shown, as is the performance of a One-Class SVM with $\mu=0.5$.}
\label{usps_ROC}\end{center}
\end{figure}
Performance was averaged over 100 randomly assigned data sets for all experiments, with performance shown in Fig. \ref{jaffe_ROC} and Fig. \ref{usps_ROC} for the JAFFE and USPS data sets, respectively. Use of the KLRR residual energy improved classification performance for simulated and real-world data sets. The ROC curves for the experiments lie above the ROC curves for either the K-nearest neighbor method or the One-Class SVM, indicating that the underlying nominal distribution likely lies on a low-dimensional manifold, and this low-dimensional structure is well approximated by the $Z$.

\bibliographystyle{IEEEtran}
%

\bibliography{IEEEabrv,min_NN_biblio}

\end{document}